%% file: elsarticle-template-num.tex
\documentclass[preprint,12pt]{elsarticle}

\usepackage{amssymb}
\usepackage{amsfonts}

\usepackage{amsmath}
\usepackage{float}
\usepackage{algpseudocode}
\usepackage{algorithm}
\usepackage{derivative}
\usepackage{comment}
\usepackage{multirow}
\usepackage{booktabs}

\usepackage{pgfplots}
\pgfplotsset{compat=1.18}
\usepackage{tikz}
\usetikzlibrary{arrows.meta,positioning,calc}

\usepackage{mathtools}

\DeclarePairedDelimiter\floor{\lfloor}{\rfloor}

\usepackage{xspace}
\usepackage{xcolor}
\usepackage{tikz}
\newcommand{\petra}[1]{{\color{brown}PB: #1}}

\renewcommand{\mp}{max-pooling\xspace}
\journal{Nuclear Physics B}

\usepackage{amsthm}
\usepackage{graphicx}
\newtheorem{theorem}{Theorem}[section]

\theoremstyle{remark}
\newtheorem{property}{Property}[section]

\begin{document}

\begin{frontmatter}



\author[label1]{Jaspreet Singh\corref{cor1}}
\ead{J.Singh@Tees.ac.uk}

\author[label2]{Petra Bosilj}
\ead{petra.bosilj@maastrichtuniversity.nl}

\author[label3]{Grzegorz Cielniak}
\ead{GCielniak@lincoln.ac.uk}

\title{Accurate Shift Invariant Convolutional Neural Networks Using Gaussian-Hermite Moments}

\cortext[cor1]{Corresponding author}
\affiliation[label1]{
    organization={School of Computing, Engineering \& Digital Technologies, Teesside University},
    city={London},
    country={UK}
}

\affiliation[label2]{
    organization={Department of Advanced Computing Sciences, Maastricht University},
    city={Maastricht},
    country={The Netherlands}
}

\affiliation[label3]{
    organization={School of Agri-food Technology and Manufacturing, University of Lincoln},
    city={Lincoln},
    country={UK}
}

\begin{abstract}
The convolutional neural networks (CNNs) are not inherently shift invariant or equivariant. The downsampling operation, used in CNNs, is one of the key reasons which breaks the shift invariant property of a CNN. Conversely, downsampling operation is important to improve computational efficiency and increase the area of the receptive field for more contextual information. In this work, we propose Gaussian-Hermite Sampling (GHS), a novel downsampling strategy designed to achieve accurate shift invariance. GHS leverages Gaussian-Hermite polynomials to perform shift-consistent sampling, enabling CNN layers to maintain invariance to arbitrary spatial shifts prior to training. When integrated into standard CNN architectures, the proposed method embeds shift invariance directly at the layer level without requiring architectural modifications or additional training procedures. We evaluate the proposed approach on CIFAR-10, CIFAR-100, and MNIST-rot datasets. Experimental results demonstrate that GHS significantly improves shift consistency, achieving 100\% classification consistency under spatial shifts, while also improving classification accuracy compared to baseline CNN models.  
\end{abstract}

\begin{graphicalabstract}
\end{graphicalabstract}

\begin{highlights}
\item Gaussian-Hermite Sampling is introduced to achieve accurate shift invariance in CNNs.

\item It provides guarantied invariance at each convolutional layer without extra training.

\item Acts as a built-in low-pass filter, reducing aliasing and noise sensitivity.

\item Achieves 100\% consistency for shifts on CIFAR-10, CIFAR-100, and MNIST-rot datasets.
\end{highlights}


\end{frontmatter}


\section{Introduction}
In computer vision, it is highly desirable for models to produce consistent outputs under affine geometric transformations such as translation, rotation, scaling, and reflection. This consistency can be expressed through two closely related but distinct concepts: \textit{invariance} and \textit{equivariance}. An operator is \textit{invariant} under a given transformation if the output remains unchanged when the input is transformed (Fig.~\ref{fig:equ_inv}, left). Invariance is key to building effective classification models where an object should be recognised regardless of its position or viewpoint. An operator is said to be \textit{equivariant} under a given transformation if applying the transformation before or after applying the operator provides equivalent results. For example, the output of a shift-equivariant operator will be spatially shifted if the input is shifted (Fig.~\ref{fig:equ_inv}, right). On the other hand, equivariance is essential for tasks that depend on precise localisation, such as object detection or semantic segmentation.

Although convolutional neural networks (CNNs) were traditionally assumed to be translation- or shift-invariant, recent studies~\cite{azulay2019deep, zhang2019making} have shown that small spatial shifts can cause considerable inconsistencies in network outputs. While convolution operations and nonlinear activations are shift-equivariant, downsampling operations such as pooling and strided convolutions disrupt this property. Nevertheless, these operations remain essential, as they reduce computational cost and the number of parameters, while enlarging the receptive field and enabling higher-level feature abstraction.

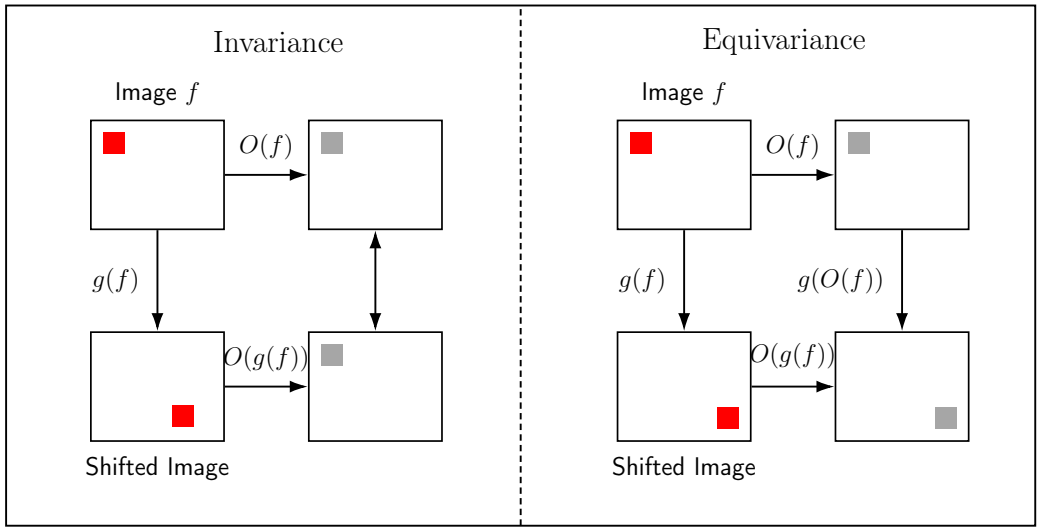
\begin{figure}[t]
\centering
\resizebox{\linewidth}{!}{%
\begin{tikzpicture}[
  font=\sffamily,
  box/.style={draw, line width=0.8pt, minimum width=2.2cm, minimum height=1.8cm},
  arr/.style={-{Latex[length=3mm,width=2mm]}, line width=0.9pt},
  darr/.style={{Latex[length=3mm,width=2mm]}-{Latex[length=3mm,width=2mm]}, line width=0.9pt},
  lab/.style={midway, fill=white, inner sep=1.2pt},
]

\begin{scope}[xshift=-7.6cm, yshift=-0.05cm]

\draw[line width=0.9pt] (-0.9,-4.25) rectangle (16.1,4.35);

\draw[densely dashed, line width=0.9pt] (7.6,-4.25) -- (7.6,4.35);

\newcommand{\fillsquare}[5]{%
  \fill[#5] ($ (#1.#2) + (#3,#4) $) rectangle ++(0.35,0.35);
}

\begin{scope}[xshift=-7.6cm]
\node[font=\large] at (11.2,3.75) {Invariance};

\node[box] (Rf)  at (9.2,1.55) {};
\node[box] (Ro)  at (12.8,1.55) {};
\node[box] (Rg)  at (9.2,-1.95) {};
\node[box] (Rgo) at (12.8,-1.95) {};

\node[above=1mm of Rf] {Image $f$};
\node[below=1mm of Rg] {Shifted Image};

\fillsquare{Rf}{north west}{0.22}{-0.57}{red}
\fillsquare{Ro}{north west}{0.22}{-0.57}{gray!70}

\fill[red] ($(Rg.south)+(0.25,0.25)$) rectangle ++(0.35,0.35);
\fillsquare{Rgo}{north west}{0.22}{-0.57}{gray!70}

\draw[arr] (Rf.east) -- (Ro.west) node[pos=0.5, above=3pt] {$O(f)$};
\draw[arr] (Rf.south) -- (Rg.north) node[pos=0.5, left=4pt] {$g(f)$};
\draw[arr] (Rg.east) -- (Rgo.west) node[pos=0.5, above=3pt] {$O(g(f))$};

\draw[darr] (Rgo.north) -- (Ro.south);
\end{scope}

\begin{scope}[xshift=+8.6cm]
\node[font=\large] at (3.35,3.75) {Equivariance};

\node[box] (Lf)  at (1.7,1.55) {};
\node[box] (Lo)  at (5.3,1.55) {};
\node[box] (Lg)  at (1.7,-1.95) {};
\node[box] (Lgo) at (5.3,-1.95) {};

\node[above=1mm of Lf] {Image $f$};
\node[below=1mm of Lg] {Shifted Image};

\fillsquare{Lf}{north west}{0.22}{-0.57}{red}
\fillsquare{Lo}{north west}{0.22}{-0.57}{gray!70}

\fillsquare{Lg}{south east}{-0.57}{0.22}{red}
\fillsquare{Lgo}{south east}{-0.57}{0.22}{gray!70}

\draw[arr] (Lf.east) -- (Lo.west) node[pos=0.5, above=3pt] {$O(f)$};
\draw[arr] (Lf.south) -- (Lg.north) node[pos=0.5, left=4pt] {$g(f)$};
\draw[arr] (Lo.south) -- (Lgo.north) node[pos=0.5, left=4pt] {$g(O(f))$};
\draw[arr] (Lg.east) -- (Lgo.west) node[pos=0.5, above=4pt] {$O(g(f))$};
\end{scope}

\end{scope} 

\end{tikzpicture}%
}
\caption{
Illustration of invariance (left) and equivariance (right).
In invariance, the operator output remains unchanged under the transformation.
In equivariance, applying the operator after a transformation yields the same result as transforming the operator output.
}
\label{fig:equ_inv}
\end{figure}

Several studies have attempted to improve shift consistency by modifying the downsampling process. Azulay and Weiss~\cite{azulay2019deep} demonstrated that anti-aliasing feature maps can mitigate sensitivity to small translations. Building on this idea, Zhang~\cite{zhang2019making} proposed combining dense \mp with strided blurred pooling to restore shift invariance. Although such low-pass filtering techniques reduce aliasing, excessive blurring can suppress discriminative details and degrade recognition accuracy. To address this, Chaman and Dokmanic~\cite{chaman2021truly} introduced Adaptive Polyphase Sampling (APS), a more robust alternative that dynamically selects stable subsampling grids across shifted inputs, preserving feature consistency under translation. APS does not inherently address aliasing effects introduced by subsampling, and is often combined with anti-aliasing filters, which sometimes lead to over-smoothing of feature maps and potential degradation in classification accuracy. Despite these advances, achieving strict and guaranteed shift invariance in CNNs remains an open challenge.

Image moments provide a principled mathematical framework for constructing invariant image representations and have been widely used in computer vision and pattern recognition \cite{zhang2004review}. Moments summarise image content by projecting it onto predefined basis functions and are broadly categorized as non-orthogonal (e.g., geometric \cite{hu1962visual} and complex moments \cite{abu1985image}) or orthogonal moments \cite{law2006image}. While non-orthogonal moments are simple to compute, they often exhibit redundancy and limited reconstruction capability. In contrast, orthogonal moments employ mutually orthogonal basis functions, yielding compact, non-redundant representations with strong discriminative power and accurate reconstruction from a finite set of coefficients. Moments can further be defined in continuous or discrete forms \cite{singh2021survey}, depending on whether their kernels are analytically defined over continuous domains or directly adapted to digital grids. Notably, several orthogonal moments-particularly those defined in the polar domain such as Zernike moments \cite{teague1980image}-naturally exhibit rotation invariance, and with appropriate normalization can also achieve translation and scale invariance. Owing to these properties, moments have long served as robust global descriptors for achieving geometric invariance in visual recognition tasks.

In this paper, we introduce a novel downsampling technique based on discrete orthogonal Gaussian–Hermite moments (GHM) \cite{yang2011image, yang2018rotation}, which are constructed using Gaussian-weighted Hermite polynomials to capture localized image structure in an orthogonal basis. The proposed downsampling technique is entirely parameter-free and does not require additional learning or training overhead. The main contributions of this work are as follows:
\begin{itemize}
    \item We propose {Gaussian–Hermite Sampling (GHS)}, the first downsampling method based on orthogonal Gaussian–Hermite polynomials.
    \item We provide a mathematical framework proving that GHS achieves exact layer-wise shift invariance without requiring additional learnable parameters.
    \item The proposed downsampling method reduces aliasing and sensitivity to noise, as the Gaussian-weighted formulation inherently acts as a low-pass filter.
    \item Experiments on CIFAR-10, CIFAR-100, and MNIST-rot demonstrate 100\% shift consistency and improved accuracy.
\end{itemize}

\section{Related work}
The equivariance and invariance properties of CNNs under geometric transformations such as translation, rotation, scaling, and reflection have been widely studied~\cite{azulay2019deep, pmlr-v48-cohenc16, marcos2017rotation, hoogeboom2018hexaconv}. It has become evident that architectures specifically designed to incorporate equivariance or invariance demonstrate significantly improved performance, even when ample training data and computational resources are available~\cite{brehmer2024does}. 
We first review methods addressing consistency under spatial shifts (Section~\ref{ssec:2_1}), followed by translation invariance in moment-based operators (Section~\ref{ssec:2_2}) and then discuss approaches extending invariance to other geometric transformations (Section~\ref{ssec:2_3}).

\subsection{Shift invariance}
\label{ssec:2_1}

Within CNNs, pooling layers are commonly employed to introduce local shift invariance. Among various pooling techniques, \mp ~\cite{ranzato2007sparse}, which selects the maximum activation within a region, and average-pooling~\cite{lecun1989handwritten}, which propagates the mean value, are the most widely used. Stochastic-pooling~\cite{zeiler2013stochastic} randomly selects an activation within each pooling region, while mixed max-average-pooling~\cite{lee2016generalizing} combines the strengths of both max- and average-pooling. Despite their popularity, these pooling operators are neither fully equivariant nor invariant to arbitrary shifts, often producing inconsistent responses when the input is slightly translated~\cite{azulay2019deep}. 

To address these limitations, several studies have investigated the underlying factors that disrupt shift consistency in CNNs. Azulay and Weiss~\cite{azulay2019deep} demonstrated that even minor input shift can lead to substantial deviations in network outputs, highlighting two primary sources of shift inconsistency: (i) aliasing introduced by improper signal sampling during downsampling operations, and (ii) non-equivalent sampling patterns arising from strided convolutions and pooling layers. These factors collectively cause CNNs to respond inconsistently to shifted inputs, even when their convolutional kernels are shift-equivariant. Although data augmentation can improve robustness to small shifts, it fails to enforce true invariance and does not address the underlying architectural causes.

To mitigate aliasing effects, Zhang~\cite{zhang2019making} proposed decomposing \mp into a dense max operation followed by subsampling, introducing a low-pass filter (LPF) between these steps to suppress high-frequency components that cause shift inconsistency. This modification significantly enhanced the shift stability of CNN outputs. Building upon this foundation, Zou et al.~\cite{zou2020delving, zou2023delving} introduced content-aware low-pass filtering, which adaptively adjusts filter strength based on local frequency characteristics to balance feature preservation and anti-aliasing. Complementarily, Chaman and Dokmanic~\cite{chaman2021truly} proposed Adaptive Polyphase Sampling (APS), a shift-invariant subsampling strategy that selects grid locations with maximal energy, thereby maintaining consistent feature representations across spatial shifts and offering a principled alternative to conventional downsampling layers.

\subsection{Translation invariance in moment-based descriptors}
\label{ssec:2_2}
Translation invariance has also been widely studied in moment-based image descriptors~\cite{chen2012quaternion}. A common strategy achieves translation invariance by shifting the origin of the image coordinate system to the centroid of the image. The centroid $(x_c, y_c)$ is computed using geometric moments of order zero and one, where $m_{00}$ denotes the zeroth-order geometric moment and $m_{10}$ and $m_{01}$ denote the first-order geometric moments along the $x$ and $y$ directions, respectively. The centroid is obtained as $x_c = m_{10}/m_{00}$ and $y_c = m_{01}/m_{00}$, and translation-invariant moments are then computed with respect to this shifted coordinate system~\cite{yang2011rotation,singh2018quaternion}. However, this normalization implicitly assumes that the image function is defined over an infinite domain and that pixels outside the image boundary are surrounded by zeros. In practical scenarios, digital images are defined on a finite support, and spatial shifts near the image boundaries introduce truncation or wrap-around effects that violate this assumption. Consequently, the centroid-based translation normalization may fail to accurately compensate for spatial shifts when objects move toward image boundaries, limiting the robustness of such moment-based translation invariance methods in real-world applications.

\subsection{General geometric equivariance and invariance}
\label{ssec:2_3}
Beyond shift, robustness to geometric transformations such as rotation and scaling is commonly pursued via data augmentation. However, augmentation does not provide intrinsic equivariance or invariance, fails to guarantee generalization to unseen transformations, and increases training cost~\cite{lafarge2021roto}. Architectures explicitly encoding geometric symmetries~\cite{pmlr-v48-cohenc16, marcos2017rotation, hoogeboom2018hexaconv} consistently outperform data augmentation, even with abundant data and computation~\cite{brehmer2024does}.

To achieve intrinsic invariance, Cohen and Welling~\cite{pmlr-v48-cohenc16} proposed the group-equivariant CNN (G-CNN), which incorporates discrete rotational and reflectional symmetries within a group-theoretic framework. Hoogeboom et al.~\cite{hoogeboom2018hexaconv} extended this concept with HexaConv, exploiting six-fold rotational symmetry for structured data. Further generalization to continuous transformations was achieved through steerable CNNs~\cite{cohen2016steerable, NEURIPS2019_45d6637b} and harmonic networks~\cite{Worrall_2017_CVPR}, which employ steerable filters to maintain equivariance across a continuous range of geometric transformations.

\section{Proposed Method}
\input{method_new.tex}

\section{Experiments} \label{sec4}
\input{experiments_new}

\section{Conclusion}
We introduced Gaussian-Hermite Sampling (GHS), a novel downsampling layer designed to achieve exact shift invariance throughout a Convolutional Neural Network (CNN). GHS utilizes 2D Gaussian–Hermite Polynomials (GHPs) as an orthogonal basis, where the built-in Gaussian-weighted mechanism acts as a natural low-pass filter to effectively eliminate aliasing during downsampling. Experiments on CIFAR-10, CIFAR-100, and MNIST-rot rigorously confirm $100\%$ consistency under input shifts and demonstrate improved robustness compared to previously established shift-invariant downsampling techniques.

The main limitation of the current GHS formulation is its computational cost, particularly for large feature maps. This overhead stems from the reliance on higher-order Gaussian-Hermite polynomials, which require evaluating computationally demanding factorial terms. Future work will focus on developing more efficient implementations, for example by exploiting recurrence relations to reduce dependence on expensive factorial computations, and on extending GHS to more complex computer vision tasks such as segmentation and object detection.

\bibliographystyle{elsarticle-num} 
\bibliography{reference}

\end{document}

%% file: method_new.tex
In this subsection, we discuss Gaussian-Hermite polynomials (GHPs) and discrete Gaussian-Hermite moments (GHMs), their fundamental mathematical properties, and introduce the proposed shift-invariant downsampling formulation.

\subsection{Gaussian-Hermite polynomials}
The continuous Hermite polynomials ~\cite{yang2011image,yang2017scale} of order $p\in \mathbb{N}_0$ (a set of non-negative integers including 0) are defined for $x\in \mathbb{R}$ as:
\begin{equation}
    H_p(x) = (-1)^p \exp{(x^2)}\odv*[order=p]{\exp(-x^2)}{x},
    \label{eq:HP}
\end{equation}
and can be explicitly rewritten in series form:
\begin{equation}
    H_p(x) = p!\sum_{k=0}^{p/2}{\frac{(-1)^k(2x)^{p-2k}}{k!(p-2k)!}}.
    \label{eq:HPS}
\end{equation}
Since, Eq.\@~\eqref{eq:HPS} is not appropriate for fast numerical calculations \cite{yang2017scale}, a 3-term recurrence relation \cite{yang2015design} is given to facilitate the fast computation of Hermite polynomials:
\begin{equation}
    \begin{split}
        H_0(x) &= 1, \\ 
        H_1(x) &= 2x, \\
        H_{p}(x) &= 2xH_{p-1}(x) - 2pH_{p-2}(x), \quad \forall p > 1.
    \end{split}
    \label{eq:HR}
\end{equation}

The Hermite polynomials, when multiplied by a Gaussian weight function, satisfy the following orthogonality condition on $\mathbb{R}$:
\begin{equation}
    \int_{-\infty} ^ {\infty} H_p(x)H_q(x)\exp(-x^2)\odif{x} = 2^p p!\sqrt{\pi}\delta_{pq}, 
    \label{eq:orthogonal}
\end{equation}
where $\delta_{pq}$ is the Kronecker delta, defined as $\delta_{pq} = [ p=q ]$ and $H_p(x)$, and $H_q(x)$ are the Hermite polynomials defined in Eq.\@~\eqref{eq:HPS}. However, it can be seen from Eq.\@~\eqref{eq:orthogonal} that Hermite polynomials are not orthonormal (i.e., not normalized to unit energy). Moreover, their amplitudes grow rapidly as $|x|$ increases, resulting in poor spatial localisation and high dynamic range. This leads to numerical instability and sensitivity to noise, making them unsuitable for direct image description \cite{yang2017scale,yang2018rotation}.To obtain orthonormal polynomials, we modulate them with a Gaussian function of scale $\sigma$, and obtain (continuous) Gaussian-Hermite polynomials (cGHP) \cite{yang2011image,yang2017scale}:
\begin{equation}
    \widehat{H}_p(x,\sigma) = \frac{1}{\sqrt{2^pp!\sigma\sqrt{\pi}}} \exp{\left(-\frac{x^2}{2\sigma^2}\right)}H_p\left(\frac{x}{\sigma}\right),
    \label{eq:GHP}
\end{equation}
which then additionally satisfy the orthonormality condition on $\mathbb{R}$:
\begin{equation}
    \int_{-\infty} ^ {\infty} \widehat{H}_p(x, \sigma)\widehat{H}_q(x, \sigma)\odif{x} = \delta_{pq}.
    \label{eq:orthonormal}
\end{equation}

\begin{figure}[t]
  \centering
   \includegraphics[width=0.7  \linewidth]{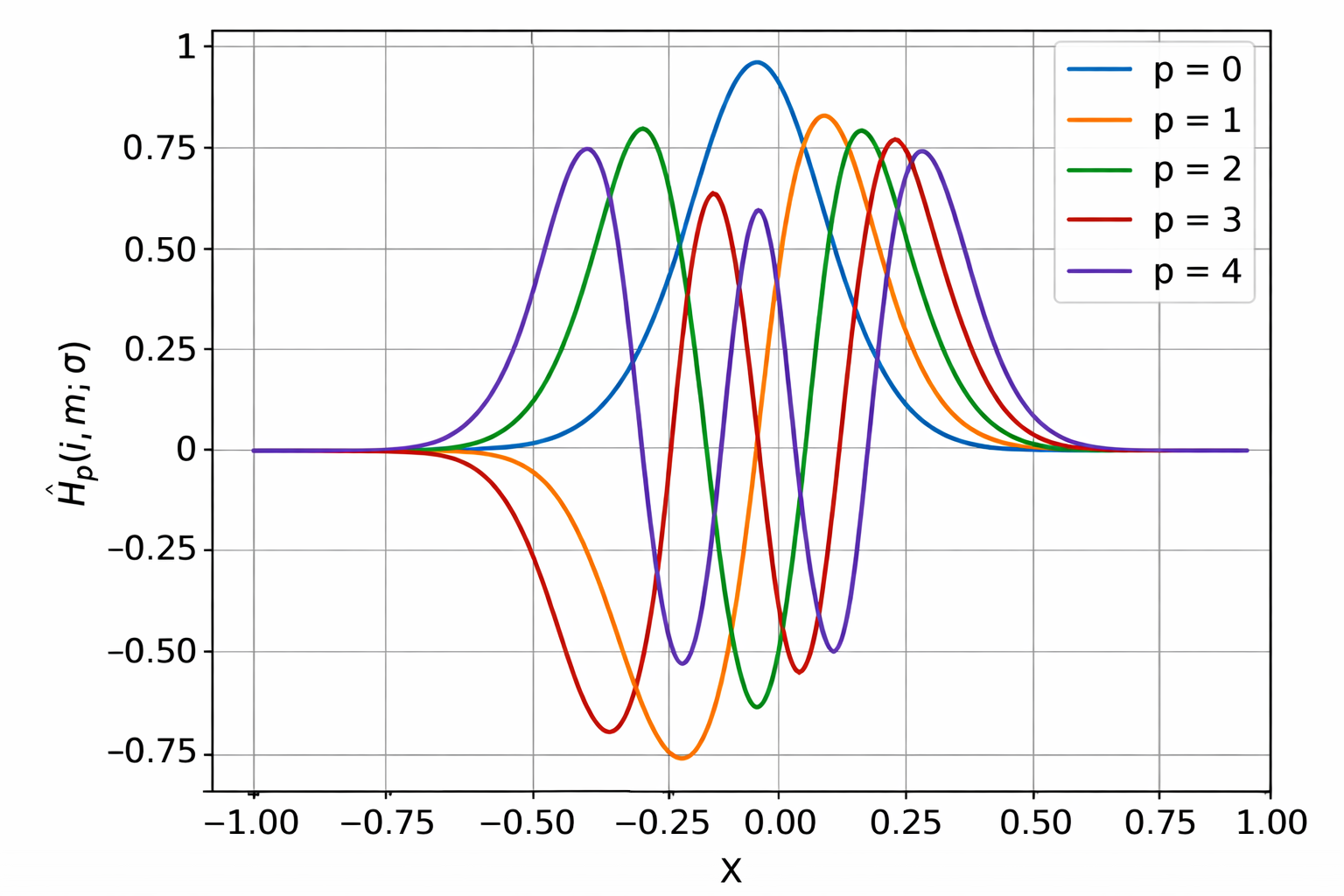}
   \caption{Gaussian-Hermite polynomials of order 0 to 4 with Gaussian weighting $(\sigma=0.75)$.}
   \label{fig:fig1}
\end{figure}

Figure \ref{fig:fig1} shows a plot of Gaussian-Hermite polynomials for order $p \in \{1,2,3,4\}$ and $\sigma=0.75$. The Gaussian weighting acts as a low-pass filter, suppressing high-frequency or noisy components in the input signal. The lower-order polynomials for $p\in\{0,1\}$ are smooth and broad, representing low-frequency components. Higher-order ones oscillate more rapidly, reflecting higher spatial frequencies. The orthonormality in Eq.\@~\eqref{eq:orthonormal} ensures that cGHPs provide a set of well-separated basis functions for decomposing the signal in the frequency domain, which enables shift-invariant and aliasing-resistant downsampling. It can also be seen that $\widehat{H}_p(x,\sigma)\in (-1, 1)$, and also that the non-zero values are effectively localised in a small neighborhood around the origin controlled by the scale parameter $\sigma$ \cite{yang2017scale}.

\subsection{Gaussian-Hermite moments}
The continuous Gaussian-Hermite moments (cGHM) \cite{yang2011image} for a 1D function $g : \mathbb{R} \to \mathbb{R}$ is defined as:
\begin{equation}
    \hat{\eta}_p = \int_{-\infty}^{\infty}\widehat{H}_p(x, \sigma)g(x)\odif{x},
    \label{cGHM1D}
\end{equation}
and can be easily extended to an arbitrary number of dimensions $n$ for $nD$ functions $g : \mathbb{R}^n \to \mathbb{R}$ as follows:
\begin{equation}
    \hat{\eta}_{p_1,\dots,p_n} = \int_{-\infty}^{\infty}\cdots\int_{-\infty}^{\infty}g(x_1, \cdots,x_n)\prod_{k=1}^{n}\widehat{H}_{p_k}(x_k, \sigma)\odif{x_k}.
    \label{eq:cGHMnD}
\end{equation}

Let $f : \{0,\dots,M-1\}\times\{0,\dots,M-1\} \to [0,1]$ define a square grayscale image of size $M\times M$ pixels. To calculate the discrete Gaussian-Hermite moments (dGHM) for a 2D image function $f$, we first define a discrete Gaussian-Hermite polynomial on $\{0, \dots, M-1\}$ as \cite{yang2011image,singh2018multi}:
\begin{equation}
    \begin{split}
    \widetilde{H}_{p^M}(i, \sigma) &= \widehat{H}_p(x_i,\sigma), \\
    \textnormal{with} \quad x_i &= \frac{2i-M+1}{M-1},
    \\
    \text{and} \quad \Delta x &= x_i - x_{i-1} = \frac{2}{M-1}.
    \end{split}
    \label{eq:dGHP}
\end{equation}
Eq.~\eqref{eq:dGHP} centralises and normalises the discrete interval $\{0, \dots, M-1\}$ to the continuous range $[-1, 1]$. This transformation preserves the essential characteristics of the cGHP, since its non-zero values are largely concentrated around the origin (see Fig.~\ref{fig:fig1}). 

We can now define the discrete Gaussian-Hermite moments (dGHM) \cite{yang2011image} for an image function $f$:
\begin{align}
    A_{pq} &= \sum_{i=0}^{M-1} \sum_{j=0} ^ {M-1} \widetilde{H}_p^M(i,\sigma)\widetilde{H}_q^M(j,\sigma)f(i,j) \Delta x \Delta y \nonumber, \\ 
    &= \frac{4}{(M-1)^2}\sum_{i=0}^{M-1} \sum_{j=0} ^ {M-1} \widetilde{H}_p^M(i,\sigma)\widetilde{H}_q^M(j,\sigma)f(i,j),
    \label{eq:dGHM} 
\end{align}
where $\Delta x$ and $\Delta y$ correspond to pixel length and with according to Eq.\@~\eqref{eq:dGHP}. Either the full or downsampled image (up to size $N\times N$, with $N \leq M$) can be reconstructed from its dGHMs of orders $(0,0)$ up to $(N-1, N-1)$ by:
\begin{equation}
    \hat{f}(i,j) = \sum_{p=0}^{N-1}\sum_{q=0}^{N-1} A_{p,q}\widetilde{H}_p^
    M(i, \sigma)\widetilde{H}_q^M(j, \sigma).
    \label{eq:reconstruction}
\end{equation}

Eqs.\@~\eqref{eq:dGHM} and \eqref{eq:reconstruction} for calculating GHMs of an image, and reconstructing the image from its GHMs, are iterative and slow and therefore to accelerate tensor processing, these equations can be reformulated into matrix form.
Let $\mathbf{F}$ be a matrix representation of image $f$ with size $M\times M$, defined as $\mathbf{F}(i,j)=f(i,j)$, and $\mathbf{Q}$ contains dGHPs of orders up to $M$ for discrete values $\{0, \dots, M-1\}$ such that $\mathbf{Q}(i,j)=\widetilde{H}_j^M(i, \sigma)$. Eq.\@~\eqref{eq:dGHM} can be rewritten such that a GHM matrix $\mathbf{A}$ (where $\mathbf{A}(i,j)=A_{i,j}$) containing all discrete moments of order $(0,0)$ up to $(M-1, M-1)$ can be computed as follows \cite{yang2011image}:
\begin{equation}
    \mathbf{A}(\mathbf{F}) = \frac{4}{(M-1)^2} \mathbf{Q}\mathbf{F}\mathbf{Q}^\top.
    \label{eq:dGHMMatrix}
\end{equation}
Similarly, Eq.\@~\eqref{eq:reconstruction} can be rewritten in matrix form as:
\begin{equation}
    \widehat{\mathbf{F}} = \mathbf{Q}^\top\mathbf{A}(\mathbf{F})\mathbf{Q}.
    \label{eq:reconstructionMatrix}
\end{equation}



Let $\mathcal{F}$ instead represent a colour image $\mathcal{F} : \{0,\dots,M-1\}\times\{0,\dots,M-1\} \to [0,1]^S$ with $S$ channels. Individual channels of this image are denoted by $f_s : \{0,\dots,M-1\}\times\{0,\dots,M-1\} \to [0,1]$ for each channel $s < S$.
$\mathcal{F}$ is then a 3D tensor, with $\mathcal{F}_{i,j,s}=\mathbf{F}_s(i,j)=f_s(i,j)$. Eqs.\@~\eqref{eq:dGHMMatrix}
 and \eqref{eq:reconstructionMatrix} still hold for calculations with colour images.

\subsection{Shift Invariance}

An operator $O(\cdot)$ is \textbf{invariant} under a geometric transformation $g(\cdot)$ if the following holds for any image $f$ (or colour image $\mathbf{\mathcal{F}}$) (see also Fig.\@~\ref{fig:equ_inv}):
\begin{equation}
    O(g(f(i,j))=O(f(i,j)).
    \label{eq:equivariance}
\end{equation}
This can be expressed for images $\mathbf{F}$ in matrix notation, where $\mathbf{O}(\cdot)$ and $\mathbf{G}(\cdot)$ represent the matrix form of the operator $O(\cdot)$ and the geometric transformation $g(\cdot)$:
\begin{equation}
    \mathbf{O(G(F))}=\mathbf{O(F)}.
    \label{eq:equivarianceMatrixForm}
\end{equation}

Let $(\cdot)^{[c,r]}$ denote the \textbf{wrap-shift} geometric transformation defined on square matrices of size $M\times M$ (image or otherwise), defined as:
\begin{equation}
    \mathbf{F}^{[c,r]}(i,j) = \mathbf{F}(i-c \pmod M, j-r \pmod M),
    \label{eq:shiftwrap}
\end{equation}
which shifts the image by $c\in\mathbb{N}$ columns and $r\in\mathbb{N}$ rows, and where $\cdot \pmod \cdot$ denotes the binary modulo operator. We note two important properties of this transformation.

\begin{property}
    For square matrices $\mathbf{A}$ and $\mathbf{B}$ of size $M\times M$, and some $a,b,c\in\mathbb{N}$, the following property of wrap-shift holds with respect to matrix multiplication:
    \begin{equation}
        \mathbf{A}^{[a,b]}\mathbf{B}^{[c,a]}=\mathbf{AB}^{[c,b]}.
    \end{equation}
    \label{prop:outer}
\end{property}

\begin{property}
    For a square matrix $\mathbf{A}$ of size $M \times M$, and some $a,b,c,d\in\mathbb{N}$, the following property holds for a chan application of wrap-shift:
    \begin{equation}
        \left(\mathbf{A}^{[a,b]}\right)^{[c,d]}=\mathbf{A}^{[a+c,b+d]}.
    \end{equation}
    \label{prop:chain}
\end{property}

Further, let $\mathbf{F}' = \mathbf{F}^{[c,r]}$ denote a (wrap-)shifted version of the image $\mathbf{F}$ by some arbitrary shift $[c,r]$, and $\mathbf{A}'$ the corresponding GHM matrix obtained using Eq.\@~\eqref{eq:dGHMMatrix}. It can be shown that the GHM matrix of the original image $\mathbf{A}$ is not equal to $\mathbf{A}'$, by decomposing Eq.\@~\eqref{eq:dGHMMatrix} into two matrix multiplication operations, $\mathbf{A}_1=\mathbf{Q}\mathbf{F}$ and $\mathbf{A}=\frac{4}{(M-1)^2}\mathbf{A}_1\mathbf{Q}^\top$. Since the matrix multiplication operation depends on the specific order of elements in each row and column, shifting the image $\mathbf{F}$ changes the result of this operation. Therefore, $\textbf{A}_1 = \textbf{Q} \textbf{F}$ and $\textbf{A}_1' = \textbf{Q} \textbf{F}^{'}$ are not equal, and consequently $\textbf{A}$ and $\textbf{A}^{'}$ are also not equal.

In order to derive a shift invariant formulation for dGHMs, we need therefore to shift the rows and columns of $\mathbf{Q}$, to adapt to the shift in rows and columns of $\mathbf{F}$. 
Specifically, we define $\mathbf{A}^*_{c,r}(\mathbf{F})$ to be the the \textbf{dGHM matrix for a pivot position} $(c,r)$, calculated as:
\begin{equation}
\mathbf{A}^*_{c,r}(\mathbf{F})
= \frac{4}{(M-1)^2}\,
\mathbf{Q}^{[r,0]} \, \mathbf{F} \, \mathbf{Q}^{\top[0,c]}.
\end{equation}
Note that the \emph{columns} of $\mathbf{Q}$ are shifted by $r$, while the \emph{rows} of $\mathbf{Q}^{\top}$ are shifted by $c$.

Finally, we define a \textbf{max-centered dGHM matrix} $\mathbf{A}^*_{\max}(\mathbf{F})$ as the dGHM matrix obtained when $(c_f, r_f)$ is used as the pivot position:
\begin{equation}
\mathbf{A}^*_{\max}(\mathbf{F}) = \mathbf{A}^*_{c_{\max},r_{\max}}(\mathbf{F}).
\label{eq:Amax}
\end{equation}

\begin{theorem}
The operator $\mathbf{A}^*_{\max}(\mathbf{F})$ is wrap-shift invariant.
\label{th:ainv}
\end{theorem}
\begin{proof}
    Let $\mathbf{F}' = \mathbf{F}^{[c,r]}$ for an arbitrary shift $(c,r)$. Then, we have $(c'_{\max}, r'_{\max}) = (c_{\max}+c \pmod M, r_{\max}+r \pmod M)$. Therefore, we can show:

\begin{align}
\mathbf{A}^*_{\max}\!\left(\mathbf{F}^{[c,r]}\right)
&= \mathbf{A}^*_{c'_{\max},r'_{\max}}(\mathbf{F}') \nonumber \\
&= \mathbf{A}^*_{c_{\max}+c,\;r_{\max}+r}\!\left(\mathbf{F}^{[c,r]}\right) \nonumber \\
&= \frac{4}{(M-1)^2}
   \mathbf{Q}^{[r_{\max}+r,0]}\,
   \mathbf{F}^{[c,r]}\,
   \mathbf{Q}^{\top[0,c_{\max}+c]} \tag{Property \ref{prop:chain}} \nonumber \\
&= \frac{4}{(M-1)^2}
   (\mathbf{Q}^{[r_{\max},0]})^{[r,0]}\,
   \mathbf{F}^{[c,r]}\,
   (\mathbf{Q}^{\top[0,c_{\max}]})^{[0,c]} \tag{Property \ref{prop:outer}} \nonumber \\
&= \frac{4}{(M-1)^2}
   (\mathbf{Q}^{[r_{\max},0]}\mathbf{F})^{[c,0]}\,
   (\mathbf{Q}^{\top[0,c_{\max}]})^{[0,c]} \tag{Property \ref{prop:outer}} \nonumber \\
&= \frac{4}{(M-1)^2}
   \mathbf{Q}^{[r_{\max},0]}\,
   \mathbf{F}\,
   \mathbf{Q}^{\top[0,c_{\max}]} \nonumber \\
&= \mathbf{A}^*_{c_{\max},r_{\max}}(\mathbf{F}) \nonumber \\
\mathbf{A}^*_{\max}\!\left(\mathbf{F}^{[c,r]}\right) &= \mathbf{A}^*_{\max}(\mathbf{F}).
\end{align}
Therefore, $\mathbf{A}^*_{\max}$ is invariant to wrap-shift according to Eq.\@~\eqref{eq:equivarianceMatrixForm}.
\end{proof}

While the existence of more than one maximum element in the image is theoretically possible, obtaining more than a single unique maximum value after a convolutional layer with $S$ channels is highly improbable. Based on our experiments, we did not observe cases where multiple sampling candidates produced identical response values, suggesting that such events are extremely rare in practice. Nevertheless, in the unlikely case, the ambiguity can be resolved by selecting the maximum within a larger neighborhood window (e.g., $3\times3$ or $5\times5$). This increases the likelihood of obtaining a unique maximum due to the larger spatial context. 

A reconstruction of the image $\mathbf{F}$ from a max-centered dGHM matrix $\mathbf{A}^*_{\max}(\mathbf{F})$ can be obtained through:
\begin{equation}
    \widehat{\mathbf{F}}^* = \mathbf{Q}^\top\mathbf{A^*_{\max}}(\mathbf{F})\mathbf{Q}.
    \label{eq:reconstructionMatrixEquiv}
\end{equation}

To obtain the downsampled reconstruction of spatial size $N\times N$, with $N\leq M$, a new $\mathbf{Q}$ is constructed of size $M\times N$.

\begin{theorem}
    The operator $\widehat{\mathbf{F}}^*$ is invariant to wrap-shifts of the input $\mathbf{F}$.
\end{theorem}
\begin{proof}
    We can show for $\mathbf{F}^{[c,r]}$ with any arbitrary shift $(c,r)$:

    \begin{align}
        \widehat{\mathbf{F}}^*\left(\mathbf{F}^{[c,r]}\right) &= \mathbf{Q}^\top\mathbf{A^*_{\max}}\left(\mathbf{F}^{[c,r]}\right)\mathbf{Q} \tag{Theorem \ref{th:ainv}} \nonumber \\
        &= \mathbf{Q}^\top\mathbf{A^*_{\max}}\left(\mathbf{F}\right)\mathbf{Q} \nonumber \\ 
    \widehat{\mathbf{F}}^*\left(\mathbf{F}^{[c,r]}\right) &= \widehat{\mathbf{F}}^*\left(\mathbf{F}\right). 
    \end{align}
    Therefore, $\widehat{\mathbf{F}}^*$ is invariant to wrap-shift according to Eq.\@~\eqref{eq:equivarianceMatrixForm}.
\end{proof}

Note that $\widehat{\mathbf{F}}^*$, being \emph{invariant} to wrap-shift, will reconstruct a \emph{wrap-shifted} version of the image $\mathbf{F}^{[M-c_{\max}, M-r_{\max}]}$ instead of the original image $\mathbf{F}$ or any $\mathbf{F}^{[c, r]}$ for arbitrary $(c, r)$.
A comprehensive, step-by-step depiction of the downsampling procedure, detailing the mechanisms of GHS, is provided in Figure \ref{fig:fig2}.

\begin{figure*}[t]
  \centering
   \includegraphics[width=1.0  \linewidth]{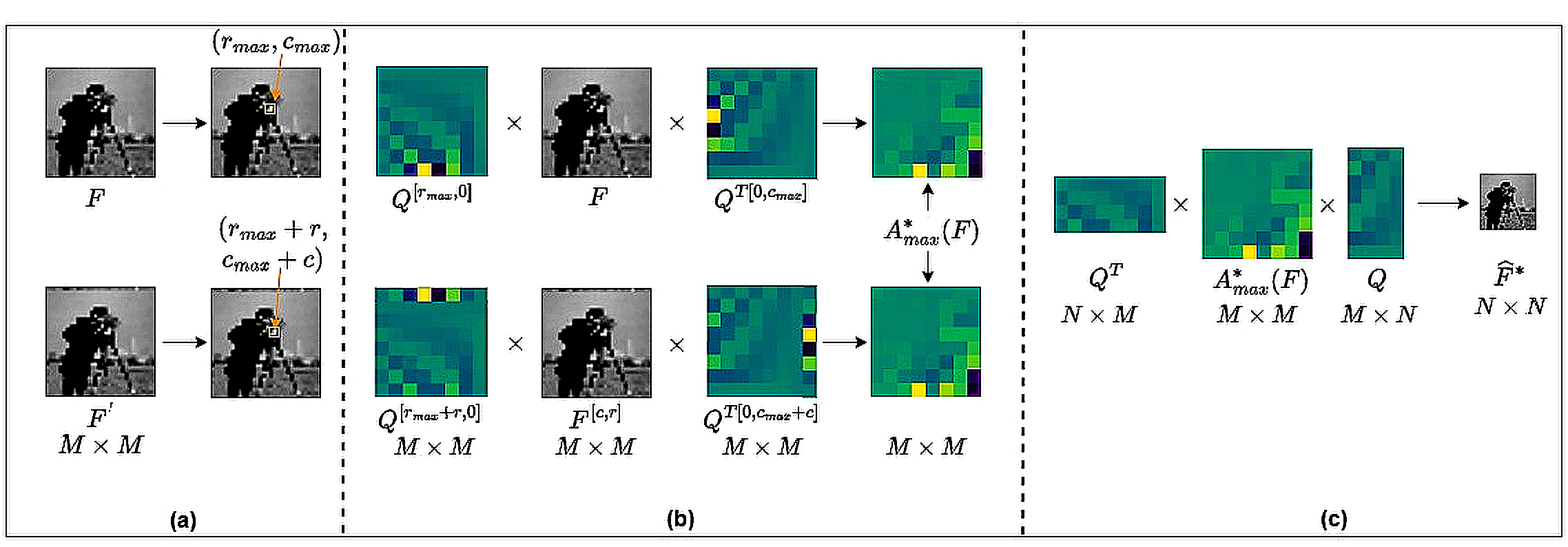}\\
\caption{
GHS process for an image. 
(a) Unshifted ($\mathbf{F}$) and shifted ($\mathbf{F}'$) images with their maximum locations. 
(b) Construction of the GHM $\mathbf{A}^*_{\max}(\mathbf{F})$ using $\mathbf{Q}$ and $\mathbf{Q}^T$. 
(c) Downsampled image $\widehat{\mathbf{F}}^*$ obtained using  (Eq.~\eqref{eq:reconstructionMatrixEquiv}).
}

   \label{fig:fig2}
\end{figure*}

\begin{figure*}[t]
\centering
\footnotesize
\begin{tabular}{ccccc}
{\includegraphics[width=2.0cm]{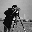}}&
{\includegraphics[width=2.0cm]{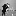}} &
{\includegraphics[width=2.0cm]{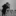}} &
{\includegraphics[width=2.0cm]{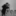}} &
{\includegraphics[width=2.0cm]{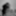}} \\
(a) Original image & (b) Max-pool & (c) LPF & (d) APS & (e) GHS \\

{\includegraphics[width=2.0cm]{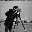}} &
{\includegraphics[width=2.0cm]{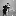}} &
{\includegraphics[width=2.0cm]{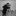}} &
{\includegraphics[width=2.0cm]{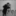}} &
{\includegraphics[width=2.0cm]{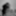}} \\
(f) Shifted image & (g) Max-pool & (h) LPF & (i) APS & (j) GHS 
\\
\end{tabular}
\caption{Comparison of downsampling methods. (a) Original image, (b-e) unshifted downsampled results, (f) shifted image and (g-j) shifted downsampled results.}
\label{fig3}
\end{figure*}

%% file: experiments_new.tex
In this section, we evaluate the performance of the proposed Gaussian-Hermite Sampling (GHS) in comparison to established benchmark approaches, including Adaptive Polyphase Sampling (APS)~\cite{chaman2021truly}, BlurPool (LPF)~\cite{zhang2019making}, and conventional max-pooling for image classification.

\subsection{Datasets}
We conduct experiments on three widely used benchmark datasets: CIFAR-10, CIFAR-100, and rotated MNIST (MNIST-rot). CIFAR-10~\cite{krizhevsky2009learning} consists of 60{,}000 colour images of size $32\times32$ across 10 object classes, with 50{,}000 images for training and 10{,}000 for testing. CIFAR-100~\cite{krizhevsky2009learning} has the same image resolution and number of images but spans 100 fine-grained classes, making it more challenging. To evaluate rotation robustness, we additionally use the MNIST-rot dataset~\cite{weiler2018learning}, which contains rotated versions of the MNIST handwritten digits, comprising 10{,}000 training images and 2{,}000 test images, each of size $28\times28$.

In addition to these datasets, we use the standard \emph{cameraman} image~\cite{gonzalez2009digital} as a test image to evaluate the numerical shift invariance. This image is used exclusively for controlled shift experiments and is not used for training or classification.

\subsection{Implementation Details}
All experiments are implemented following the training configuration of Chaman and Dokmanic~\cite{chaman2021truly} and Zhang~\cite{zhang2019making}, using stochastic gradient descent (SGD) with a batch size of 256, momentum of 0.9, and a weight decay of $5\times10^{-4}$. The typical values of $\sigma$ used in Eq.\@~\eqref{eq:reconstruction} are calculated as follows \cite{yang2011image}:
\begin{equation}
    \sigma = \begin{cases} 
0.9(N)^{-0.52} & \text{if } N \geq 1, \\
1.0 & \text{if } N = 0.
\end{cases}
\end{equation}

The experiments are conducted by replacing all pooling layers in a CNN model by the proposed GHS layers. Specifically, we experiment using four ResNet variants: ResNet-18, ResNet-20, ResNet-50, and ResNet-56. The initial learning rate is set to 0.1 and reduced by a factor of 0.1 at epochs 100 and 200, for a total of 250 training epochs. All experiments are repeated over 10 random seeds, and results are reported as mean $\pm$ standard deviation. For fair comparison, both LPF and APS use the same 3-tap anti-aliasing filter as in their original implementations.

\subsection{Evaluation Metrics and Setup}

We evaluate the proposed and baseline methods using classification accuracy under multiple evaluation settings, including standard test data, out-of-distribution (OOD) perturbations, and rotated inputs.

\noindent\textit{Classification Accuracy:}
Classification accuracy is defined as the proportion of correctly classified test samples. Unless otherwise specified, accuracy is reported on unshifted test images.

\noindent\textit{Evaluation under OOD Perturbation:}
To assess robustness to distributional shifts, we introduce controlled perturbations to the test images, including patch erasures and vertical flips. Classification accuracy and prediction consistency are reported on the perturbed data.

\noindent\textit{Evaluation under Rotation:}
Rotation robustness is evaluated on MNIST-rot using an $E(2)$-equivariant CNN. Since MNIST-rot labels are rotation-invariant, classification accuracy is measured on rotated test images. Prediction consistency between original and rotated inputs is also reported.

\subsection{Experiments and Results}
In this subsection, we present experiments evaluating shift invariance, followed by classification accuracy, robustness to OOD perturbations, and rotation-invariant classification performance.\\

\noindent\textit{Shift Invariance:}
Numerical shift invariance is evaluated using the cameraman image ~\cite{gonzalez2009digital}, which enables controlled analysis of downsampling behavior under spatial shifts. For an image $f$ and its shifted version $f'$, a downsampling operator $O$ is assessed by computing the absolute difference (AD) between their downsampled outputs. Lower AD values indicate stronger shift invariance. A single-channel $32 \times 32$ image is downsampled using each method. Figure~\ref{fig3} illustrates the original and shifted inputs together with their corresponding downsampled outputs. The resulting AD values are: max-pooling (2.77), LPF (0.70), APS (0.48), and GHS (0.00), confirming that GHS achieves perfect shift invariance at the individual layer level and eliminates the need for global pooling to restore consistency.\\

\noindent\textit{Classification Performance:}
To evaluate the classification performance of the proposed downsampling layer, we report classification consistency and accuracy on CIFAR-10 using four ResNet variants, and on CIFAR-100 using ResNet-18. Table~\ref{table:consistency} reports the classification consistency results on CIFAR-10 across all ResNet variants, while Table~\ref{table:accuracy} presents the corresponding accuracy values. Both APS and GHS achieve 100\% classification consistency on all models, whereas LPF shows moderate improvement and the baseline exhibits the weakest robustness. In terms of accuracy (Table~\ref{table:accuracy}), GHS achieves the highest performance, reaching 95.54\% on ResNet-18. Results on CIFAR-100 dataset are shown in Table~\ref{table_2} which follow a similar pattern: GHS and APS attain 100\% consistency, with GHS outperforming APS in accuracy (77.92\% vs.\ 77.12\%). These findings demonstrate that GHS provides both theoretical and empirical advantages in achieving shift-invariant behavior.\\

\begin{table*}[t]
\footnotesize   
\caption{Classification consistency on CIFAR-10 test set using ResNet models with baseline, LPF, APS, and GHS downsampling methods. Best results are highlighted in bold.}
\vspace{0.5em}
\label{table:consistency}
\centering
{%
\begin{tabular}{c|cccc}
\hline
\textbf{Method} & \textbf{ResNet-18} & \textbf{ResNet-20} & \textbf{ResNet-50} & \textbf{ResNet-56} \\
\hline
Baseline & 90.88$\pm{0.16}$ & 90.83$\pm{0.14}$ & 88.96$\pm{0.14}$ & 91.89$\pm{0.31}$ \\
LPF      & 98.10$\pm{0.12}$ & 96.53$\pm{0.13}$ & 97.38$\pm{0.20}$ & 96.90$\pm{0.19}$ \\
APS      & 100.00$\pm{0.00}$    & 100.00$\pm{0.00}$    & 100.00$\pm{0.00}$    & 100.00$\pm{0.00}$ \\
\textbf{GHS} & \textbf{100.00}$\pm{0.00}$ & \textbf{100.00}$\pm{0.00}$ & \textbf{100.00}$\pm{0.00}$ & \textbf{100.00}$\pm{0.00}$ \\
\hline
\end{tabular}%
}
\end{table*}

\begin{table*}[t]
\footnotesize   
\caption{Classification accuracy on CIFAR-10 test set using ResNet models with baseline, LPF, APS, and GHS downsampling methods. Best results are highlighted in bold.}
\vspace{0.5em}
\label{table:accuracy}
\centering
{%
\begin{tabular}{c|cccc}
\hline
\textbf{Method} & \textbf{ResNet-18} & \textbf{ResNet-20} & \textbf{ResNet-50} & \textbf{ResNet-56} \\
\hline
Baseline & 91.96$\pm{0.31}$ & 89.76$\pm{0.26}$ & 90.05$\pm{0.18}$ & 91.40$\pm{0.20}$ \\
LPF      & 94.28$\pm{0.29}$ & 91.56$\pm{0.30}$ & 94.12$\pm{0.32}$ & 92.98$\pm{0.23}$ \\
APS      & 94.48$\pm{0.17}$ & \textbf{91.75}$\pm{0.23}$ & 94.07$\pm{0.25}$ & 92.93$\pm{0.29}$ \\
\textbf{GHS} & \textbf{95.54}$\pm{0.13}$ & 91.72$\pm{0.18}$ & \textbf{94.40}$\pm{0.17}$ & \textbf{93.13}$\pm{0.20}$ \\
\hline
\end{tabular}%
}
\end{table*}

\begin{table}[t]
\centering
\footnotesize   
\caption{Consistency and accuracy on the CIFAR-100 test set using the ResNet-18 model.}
\vspace{0.5em}
\begin{tabular}{c|c|c}
\hline
\textbf{Method} & \textbf{Consistency} & \textbf{Accuracy} \\
\hline
Baseline & 84.35 & 75.59 \\
LPF      & 90.54 & 76.85 \\
APS      & 100.00   & 77.12 \\
\textbf{GHS} & \textbf{100.00} & \textbf{77.92} \\
\hline
\end{tabular}
\label{table_2}
\end{table}

\begin{table*}[t]
\centering
\footnotesize   
\caption{Consistency and accuracy using ResNet-18 on unflipped and flipped CIFAR-10.}
\vspace{0.5em}
\begin{tabular}{c|cc|cc}
\hline
\multirow{2}{*}{\textbf{Method}} 
& \multicolumn{2}{c|}{\textbf{Consistency}} 
& \multicolumn{2}{c}{\textbf{Accuracy}} \\
\cline{2-5}
& Unflip & Flip & Unflip & Flip \\
\hline
DA & 97.84 & 84.94 & 94.22 & 44.97 \\
LPS & 98.19 & 89.21 & 94.28 & 46.21 \\
APS & 100.00 & 100.00 & 94.48 & 47.55 \\
GHS & 100.00 & 100.00 & \textbf{95.02} & \textbf{49.68}\\
\hline
\end{tabular}
\label{table_3}
\end{table*}



\begin{figure*}[t]
\centering

\begin{minipage}{0.49\linewidth}
\centering
\begin{tikzpicture}
\begin{axis}[
    width=\linewidth,
    height=6cm,
    ymin=90, ymax=100.5,
    ytick={90,92,94,96,98,100},
    ylabel={Classification consistency (\%)},
    xlabel={Size of randomly erased square patch},
    symbolic x coords={0,4,6,8},
    label style={font=\small},   
    xtick=data,
    ybar,
    bar width=6pt,
    enlarge x limits=0.18,
    axis line style={line width=1pt},
    tick style={line width=1pt},
    tick align=outside,     
    tick pos=left,          
    xtick pos=bottom,       
    legend style={at={(0.5,-0.30)}, anchor=north, legend columns=4, draw=none},
]

\addplot[fill=cyan!70, draw=cyan!70]
coordinates {(0,100) (4,100) (6,100) (8,100)}; 

\addplot[fill=orange!85, draw=orange!85]
coordinates {(0,100) (4,100) (6,100) (8,100)}; 

\addplot[fill=green!50, draw=green!50]
coordinates {(0,98.1) (4,97.8) (6,97.0) (8,95.0)}; 

\addplot[fill=magenta!65, draw=magenta!65]
coordinates {(0,97.9) (4,97.1) (6,96.6) (8,95.2)}; 

\legend{GHS, APS, LPF, DA}
\end{axis}
\end{tikzpicture}
\end{minipage}
\hfill
\begin{minipage}{0.49\linewidth}
\centering
\begin{tikzpicture}
\begin{axis}[
    width=\linewidth,
    height=6cm,
    ymin=86, ymax=96,
    ytick={86,88,90,92,94,96},
    ylabel={Classification accuracy (\%)},
    xlabel={Size of randomly erased square patch},
    symbolic x coords={0,4,6,8},
    label style={font=\small},   
    xtick=data,
    ybar,
    bar width=6pt,
    enlarge x limits=0.18,
    axis line style={line width=1pt},
    tick style={line width=1pt},
    tick align=outside,     
    tick pos=left,          
    xtick pos=bottom,       
    legend style={at={(0.5,-0.30)}, anchor=north, legend columns=4, draw=none},
]

\addplot[fill=cyan!70, draw=cyan!70]
coordinates {(0,94.9) (4,93.8) (6,91.8) (8,89.6)}; 

\addplot[fill=orange!85, draw=orange!85]
coordinates {(0,94.7) (4,93.7) (6,91.7) (8,89.5)}; 

\addplot[fill=green!50, draw=green!50]
coordinates {(0,94.2) (4,93.5) (6,91.5) (8,88.6)}; 

\addplot[fill=magenta!65, draw=magenta!65]
coordinates {(0,94.1) (4,93.4) (6,91.4) (8,87.8)}; 

\legend{GHS, APS, LPF, DA}
\end{axis}
\end{tikzpicture}
\end{minipage}

\caption{Consistency (left) and accuracy (right) of ResNet-18 with different downsampling methods on CIFAR-10 under random patch erasures.}
\label{fig:patch_robustness}
\end{figure*}
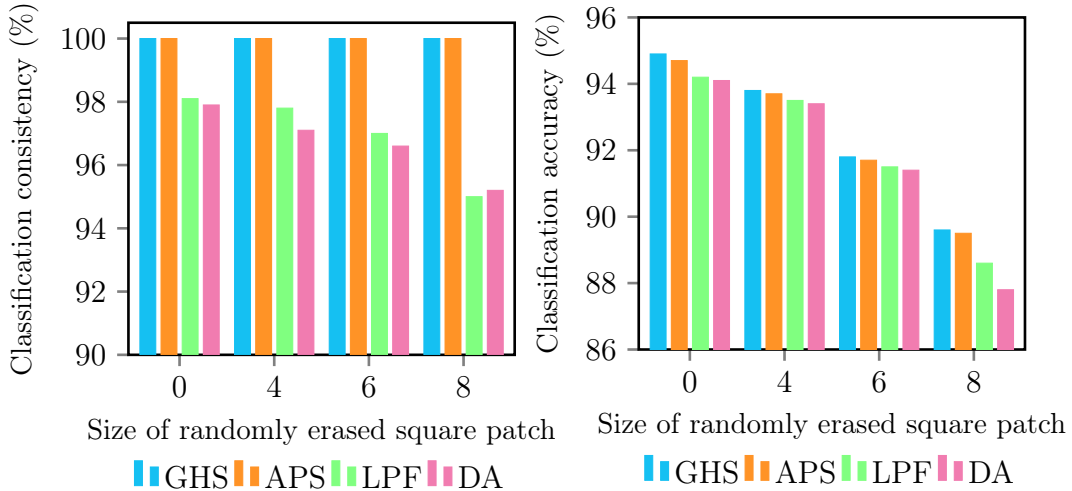

\noindent\textit{Out-of-Distribution Evaluation:}
We further evaluate robustness on perturbed CIFAR-10 test images by introducing controlled distributional shifts, including random patch erasures and vertical flips. Specifically, random square patches of sizes 0, 4, 6, and 8 pixels are erased from the images. Models trained using standard, LPF, APS, and GHS downsampling are tested under these perturbations. For patch erasures of varying sizes, shown in Fig.~\ref{fig:patch_robustness}, GHS maintains 100\% consistency and achieves the highest classification accuracy, demonstrating superior generalisation to unseen data. Similar trends are observed for vertically flipped CIFAR-10 images, as reported in Table~\ref{table_3}, where GHS preserves 100\% consistency and improves accuracy by 2.1\% over APS.\\

\noindent\textit{Rotation Invariance Classification:}
To evaluate robustness under rotations, we integrate GHS into a $E(2)$-equivariant CNN~\cite{weiler2019general} and evaluate performance on the MNIST-rot dataset. This setting allows us to assess whether the proposed downsampling layer can be effectively combined with group-equivariant architectures designed to handle rotational symmetries. As shown in Table~\ref{table_4}, GHS achieves the highest performance among all methods, outperforming both APS and LPF, and demonstrating that GHS preserves its advantages when incorporated into rotation-equivariant models.

\subsection{Complexity Analysis}
Table~\ref{table_6} reports the computational overhead of different pooling and subsampling methods measured over 100 batches (batch size $32$, input size $3 \times 32 \times 32$) on a single NVIDIA T4 GPU. As expected, baseline max pooling is the most efficient. LPF introduces a modest overhead of approximately $50\%$ ($\sim1.5\times$) due to the additional filtering operation. In contrast, APS and GHS are roughly an order of magnitude ($\sim10\times$) more expensive than baseline max pooling. However, GHS adds only a small overhead relative to APS. Given that APS is already widely adopted, this marginal increase is unlikely to hinder the practical use of GHS, which remains parameter-free while providing strict shift-invariance guarantees.
\begin{table}[t]
\footnotesize   
\centering

\begin{minipage}{0.48\linewidth}
\centering
\caption{Accuracy on the rotated MNIST test set using an $E(2)$-equivariant model with different downsampling methods.}
\begin{tabular}{c|c}
\hline
\textbf{Method} & \textbf{Accuracy} \\ 
\hline
Baseline & 96.69 \\
LPF      & 97.85 \\
APS      & 98.12 \\
\textbf{GHS} & \textbf{98.82} \\
\hline
\end{tabular}
\label{table_4}
\end{minipage}
\hfill
\begin{minipage}{0.48\linewidth}
\centering
\caption{Computational time for Max pool, LPF, APS, and GHS measured over 100 batches, each with batch size 32 and image size $3\times 32 \times 32$.}
\begin{tabular}{c|c}
\hline
\textbf{Method} & \textbf{Time (ms)} \\
\hline
Baseline (Max pool) & $\mathbf{0.27\pm0.10}$ \\
LPF & $0.41\pm0.13$ \\
APS & $2.56\pm0.20$ \\
\textbf{GHS} & $2.71\pm0.12$ \\
\hline
\end{tabular}
\label{table_6}
\end{minipage}

\end{table}